\documentclass[twoside]{article}

\newcommand \longver 0
 \usepackage[accepted]{aistats2020}
\usepackage[round]{natbib}

\allowdisplaybreaks[1]
\usepackage{graphicx}
\usepackage{subcaption}  % alternative to caption package is subfig , which is succcesor to subfigure pkg. But subfig also uses caption package,
\usepackage{booktabs,enumitem}       % professional-quality tables
\usepackage{amsfonts}      % blackboard math symbols
\usepackage{nicefrac}       % compact symbols for 1/2, etc.
\usepackage{amsmath,amssymb}
\usepackage{tabularx}
\usepackage{algorithm,algorithmic}
\usepackage{tikz}
\usetikzlibrary{automata}
\usetikzlibrary{topaths}
\usetikzlibrary{shapes}
\usetikzlibrary{arrows}
\usetikzlibrary{decorations.markings}
\usetikzlibrary{intersections}
\usetikzlibrary{backgrounds}

\tikzstyle{utility}=[diamond,draw=black,draw=blue!50,fill=blue!10,inner sep=0mm, minimum size=8mm]
\tikzstyle{select}=[rectangle,draw=black,draw=blue!50,fill=blue!10,inner sep=0mm, minimum size=6mm]
\tikzstyle{hidden}=[dashed,draw=black,fill=red!10]
\tikzstyle{RV}=[circle,draw=black,draw=blue!50,fill=blue!10,inner sep=0mm, minimum size=6mm]
\tikzstyle{place}=[circle,draw=black,draw=blue!50,fill=blue!20,inner sep=0mm, minimum size=9mm]

\newcommand{\pluseq}{\mathrel{+}=}

\renewcommand \Pr {\mathop{\mathbb{P}}\nolimits}
\newcommand \argmax {\mathop{\arg \max}}

\newcommand \bel {\beta}
\newcommand \mdp {\mu}

\newcommand \pol {\pi}
\newcommand \generator {\mathcal{P}}

\newcommand \TSpol {\pol^{\textrm{TS}}_{\bel}}

\newcommand \kBOpol {\pol_{\bel}^{\steps}}

\newcommand \BOpol {\pol_{\bel}^{*}}

\newcommand \DSpol {\pol^{\textrm{DS}}_{\bel}}

\newcommand \TSutil {\util^{\textrm{TS}}_{\bel}}
\newcommand \TSutilH {\util^{\textrm{TS}}_{\bel_h}}
\newcommand \kBOutil {\util_{\bel}^{\steps}}
\newcommand \kBOutilH {\util_{\bel_h}^{\steps}}
\newcommand \BOutil {\util_{\bel}^{*}}
\newcommand \BOutilH {\util_{\bel_h}^{*}}
\newcommand \DSutil {\util^{\textrm{DS}}_{\bel}}

\newcommand \disc {\gamma}
\newcommand \util {v}

\newcommand \defn {\mathrel{\triangleq}}

\newcommand \cset[2]{\left\{#1 ~\middle|~ #2 \right\}}
\newcommand \norm[1]{\left\|#1\right\|}
\newcommand \pnorm[2]{\norm{#1}_{#2}}
\newcommand \MDPdist[2]{D(#1, #2)}

\newcommand \horizon {T}
\newcommand \nsamples {M}
\newcommand \npolicies {N}
\newcommand \steps {K}
\newcommand \finalStage {H}

\newcommand{\expect}{\mathbb{E}}

\usepackage{amsthm,amsfonts}
\theoremstyle{plain}
\newtheorem{theorem}{Theorem}
\newtheorem{lemma}{Lemma}

\newtheorem{definition}{Definition}
\newtheorem{assumption}{Assumption}
\theoremstyle{remark}

\makeatletter
\newtheorem*{rep@theorem}{\rep@title}
\newcommand{\newreptheorem}[2]{%
	\newenvironment{rep#1}[1]{%
		\def\rep@title{\textbf{#2 \ref{##1}}}%
		\begin{rep@theorem}}%
		{\end{rep@theorem}}}
\makeatother
\newreptheorem{theorem}{Theorem}
\newreptheorem{lemma}{Lemma}

\newcommand \linf[1]{\left\|#1\right\|_{\infty}}
\def\clap#1{\hbox to 0pt{\hss#1\hss}}

\def\mathclap{\mathpalette\mathclapinternal}

\def\mathclapinternal#1#2{%
           \clap{$\mathsurround=0pt#1{#2}$}}

\begin{document}
\runningauthor{Divya Grover, Debabrota Basu, Christos Dimitrakakis}
	% \runningtitle{I use this title instead because the last one was very long}

	% \runningauthor{Surname 1, Surname 2, Surname 3, ...., Surname n}

	\twocolumn[

	\aistatstitle{Bayesian Reinforcement Learning via Deep, Sparse Sampling}
	%\aistatsauthor{ Anonymous}
	%\aistatsaddress{Institute}
	%]
	\aistatsauthor{ Divya Grover$^1$ \And Debabrota Basu$^1$ \And  Christos Dimitrakakis$^{1,2}$ }
	\aistatsaddress{Chalmers University of Technology$^1$ \And University of Oslo$^2$}
	]

\begin{abstract}
We address the problem of Bayesian reinforcement learning using efficient model-based online planning. We propose an optimism-free Bayes-adaptive algorithm to induce deeper and sparser exploration with a theoretical  bound on its performance relative to the Bayes optimal as well as lower computational complexity. The main novelty is the use of a candidate policy generator, to generate long-term options in the planning tree (over beliefs), which allows us to create much sparser and deeper trees. Experimental results on different environments show that in comparison to the state-of-the-art, our algorithm is both computationally more efficient, and obtains significantly higher reward over time in discrete environments.
\end{abstract}

\section{INTRODUCTION}
In Reinforcement Learning~\citep{sutton1998rl}, an agent sequentially interacts with an unknown environment with the objective of maximising its total reward over time. As the environment is unknown to the agent, it must carefully balance its actions in order to learn more about the environment  (\emph{exploration}) and obtain reward with high certainty (\emph{exploitation}) as well. This dilemma of balancing exploration in the environment with exploiting the existing knowledge is referred to as the \emph{exploration--exploitation trade-off}.

\emph{Bayesian Reinforcement Learning} (BRL) solves this trade-off  by constructing and using a probability distribution over possible models of the environment and trying to maximise total reward in expectation while marginalising over all possible models. This automatically takes into account the uncertainty about the environment. However, this ``Bayes-optimal'' policy is generally intractable as it requires performing dynamic programming over an exponentially large tree. Simpler solutions, such as Thompson sampling~\citep{thompson1933lou}, are known to be nearly optimal in some settings, such as multi-armed bandits~\citep{Kaufmann:Thompson}. Alternatively, one can construct approximate versions of the planning tree through Monte Carlo rollouts, sparse sampling, and limited look-ahead~\citep{dimitrakakis2013monte, castro2010smarter, guez2012efficient}.

	In this paper, we introduce the DSS (Deeper Sparser Sampling)
	algorithm to alleviate problems with existing approximations of the
	Bayes-optimal planner. DSS uses \emph{policy samples} to create a
	deep tree with a smaller branching factor. We show that at any step,
	our algorithm produces an action that is with high probability close
	to the Bayes-optimal, and demonstrate experimentally that it
	outperforms the state-of-the-art BRL methods with significantly less
	computation.

	The rest of the paper is organised as follows. In Section~\ref{sec:method}, we describe the framework of Markov Decision Processes (MDP) and Bayesian reinforcement learning. In Section~\ref{sec:lit}, we discuss related work and the outline of our contribution. Section~\ref{sec:algo} elaborates the DSS algorithm. Then, we follow up by theoretical and experimental analysis of DSS in Section~\ref{sec:theory} and~\ref{sec:experiment} respectively. Some technical proofs are relegated to the Appendix.

	\begin{figure*}[t!]
		\centering
		\begin{subfigure}{0.5\textwidth}%\vspace*{-1em}
			\centering
			\begin{tikzpicture}
			\node[RV] at (0,0) (st) {$\omega_t$};
			\node[select] at (1,-1) (at) {$a_t^1$};
			\node[select] at (1,1) (at2) {$a_t^2$};
			\draw[->] (st) to (at);
			\draw[->] (st) to (at2);
			\node[RV] at (2,-1.5) (s2t1) {$\omega_{t+1}^{11}$};
			\draw[dashed, ->] (s2t1) to (3,-2);
			\node[RV] at (2,-0.5) (s2t2) {$\omega_{t+1}^{12}$};
			\node[RV] at (2,0.5) (s2t3) {$\omega_{t+1}^{21}$};
			\node[RV] at (2,1.5) (s2t4) {$\omega_{t+1}^{22}$};
			\draw[dashed, ->] (s2t4) to (3,2);
			\draw[->] (at) to (s2t1);
			\draw[->] (at) to (s2t2);
			\draw[->] (at2) to (s2t3);
			\draw[->] (at2) to (s2t4);
			\node[select] at (3,1) (a3t3) {$a_{t+1}^1$};
			\draw[dashed, ->] (a3t3) to (4,2);
			\node[select] at (3,0) (a3t1) {$a_{t+1}^2$};
			\node[select] at (3,-1) (a3t2) {$a_{t+1}^3$};
			\draw[->] (s2t2) to (a3t1);
			\draw[->] (s2t2) to (a3t2);
			\draw[->] (s2t3) to (a3t3);
			\node[RV] at (4,-2) (s3t1) {$\omega_{t+2}^{31}$};
			\node[RV] at (4,-1) (s3t2) {$\omega_{t+2}^{32}$};
			\node[RV] at (4,0) (s3t3) {$\omega_{t+2}^{21}$};
			\node[RV] at (4,1) (s3t4) {$\omega_{t+2}^{22}$};
			\draw[->] (a3t2) to (s3t1);
			\draw[->] (a3t2) to (s3t2);
			\draw[->] (a3t1) to (s3t3);
			\draw[->] (a3t1) to (s3t4);
			\end{tikzpicture}
			\caption{Full tree expansion.}\label{fig:tree1}
		\end{subfigure}%\vspace*{-.5em}\hspace*{4em}
		\begin{subfigure}{0.5\textwidth}
			\centering
			\begin{tikzpicture}
			\node[RV] at (0,0) (st) {$\omega_t$};
			\node[select] at (1,-1) (at) {$\pi_t^1$};
			\node[select] at (1,1) (at2) {$\pi_t^2$};
			\draw[->] (st) to (at);
			\draw[->] (st) to (at2);
			\node[RV] at (4,-1.5) (s2t1) {$\omega_{t+K}^{11}$};
			\node[RV] at (4,-0.5) (s2t2) {$\omega_{t+K}^{12}$};
			\node[RV] at (4,0.5) (s2t3) {$\omega_{t+K}^{21}$};
			\node[RV] at (4,1.5) (s2t4) {$\omega_{t+K}^{22}$};
			\draw[->] (at) to (s2t1);
			\draw[->] (at) to (s2t2);
			\draw[->] (at2) to (s2t3);
			\draw[->] (at2) to (s2t4);
			\end{tikzpicture}%\vspace*{-.5em}
			\caption{Deeper $\&$ Sparser tree expansion.} \label{fig:tree2}
		\end{subfigure}%\vspace*{-.5em}\\
		\caption{Visualising tree expansion. $\omega_{t}^{ij}$ denotes the information state at time $t$ given action $i$ and having observed state $j$.}%\vspace*{-.5em}
	\end{figure*}
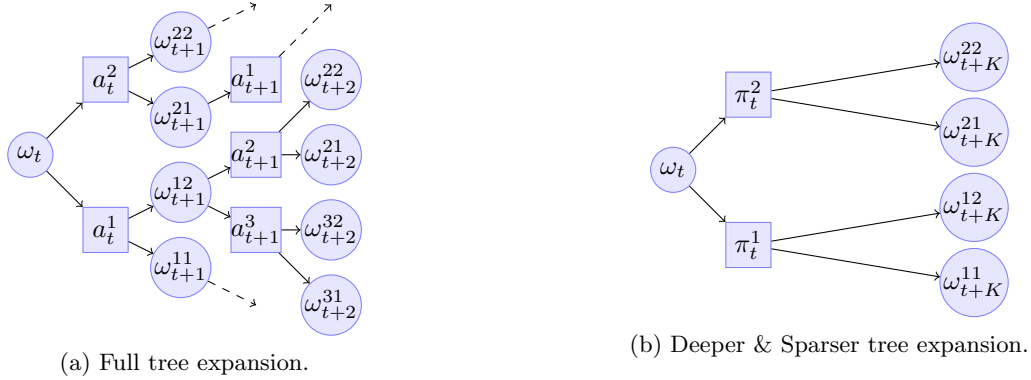

	%%%%%%%%%%%%%%%%%
	\section{BACKGROUND AND RELATED WORK}\label{sec:method}
	\subsection{Markov Decision Process (MDP)}
	Markov Decision Process (MDP) is a discrete-time stochastic process that provides a formal framework for reinforcement learning problems.
	\begin{definition}
		An MDP $\mu = (S,A,P,R)$ is composed of a state space $S$, an action space $A$, a reward distribution $R$ and a transition function $P$. The transition function $P \defn \Pr_\mu(s_{t+1}|s_t ,a_t )$ dictates the distribution over next states $s_{t+1}$ given the present state-action pair $(s_t, a_t)$. The reward distribution $R \defn \Pr_\mu(r_{t+1}|s_t ,a_t)$  dictates the obtained reward with $r \in [0,1]$.
		% Thus, $\prob(r_{t+1}, s_{t+1}|s_t ,a_t ) = R(r_{t+1}|s_t ,a_t )P(s_{t+1}|s_t ,a_t)$.
		We shall also use $\Pr_\mdp(r_{t+1}, s_{t+1}|s_t ,a_t )$ to denote the joint distribution of next states and actions of MDP $\mdp$.
	\end{definition}
	A policy $\pol$ belonging to a policy space $\Pi$ is an algorithm for selecting actions given present state and previous observations. A policy is Markov if at any time $t$, the action $a_t \in A$ chosen by the policy only depends on the current state $s_t$, so that the action distribution can be written as $\pol_t(a_t \mid s_t)$.
	\\
	The value function of a policy for a specific MDP is the expected sum of discouted rewards obtained from time $t$ to $T$ while selecting actions in the MDP $\mdp$:
	\begin{align}
	% \[
	V_{t,T}^{\pi,\mdp}(s)
	=
	\expect^\pi_\mdp ( \sum_{k=1}^{T} \disc^k r_{t+k}  \mid s_t = s), \label{eqn:val}
	% \]
	\end{align}
	where $\gamma \in (0, 1]$ is called the discount factor and $\expect^\pi_\mdp$ denotes the expectation under the Markov chain generated by a policy $\pi$ acting on the MDP $\mdp$. Let us define the infinite horizon discounted value function of a policy $\pi$ on an MDP $\mu$ as $ V_\mdp^\pi \defn \lim_{T\to\infty} V_{0,T}^{\pi,\mdp}$. Now, we define the \emph{optimal value function} to be $V_\mdp^* \triangleq \; \max_\pi V_\mdp^\pi$, and the \textit{optimal policy}  to be $\pol^*_\mdp \defn \argmax_\pi V_\mdp^\pi$. If the MDP is known, the optimal policy and value function is computable via backwards induction (alias, value iteration)~\citep{Puterman:MDP:1994}. %Whenever it is clear from context, superscripts and subscripts shall be omitted for brevity.

	\subsection{Bayes Adaptive MDP (BAMDP)}
	In reality, the underlying MDP is unknown to the reinforcement learning algorithm. This amounts to a trade-off between information seeking actions for performing better exploration and acting optimally given the current knowledge i.e. exploitation. This exploration-exploitation trade-off is one of the central issues in reinforcement learning. Bayesian Reinforcement Learning (BRL), specifically the information state formulation~\citep{dearden99bayesian,duff2002olc}, provides a framework to quantify this trade-off using Bayesian representation.
	\\
	% Following the ideas from \cite{duff2002olc}, we use a Bayesian framework to represent our uncertainty.
	Following the Bayesian formulation, we maintain a belief distribution $\bel_t$  over the possible MDP models $\mdp \in \mathcal{M}$.\footnote{More precisely, we can define a measurable space $(\mathcal{M},\mathfrak{M})$, where $\mathcal{M}$ is the possible set of MDPs, and $\mathfrak{M}$ is a suitable $\sigma$-algebra.} With an appropriate prior belief $\bel_0(\mdp)$, we obtain a sequence of posterior beliefs $\bel_t (\mdp)$ that represents our subjective belief over the MDPs at time $t$, depending on the latest observations. By Bayes' rule, the posterior belief at time $t+1$ is
	\begin{align}
	\bel_{t+1}(\mdp) &\triangleq \frac{\Pr_\mdp(r_{t+1}, s_{t+1}|s_t ,a_t)\bel_t(\mdp)}{\int_{\mathcal{M}} \Pr_\mdp^{'}(r_{t+1}, s_{t+1}|s_t ,a_t)\bel_t(\mdp^{'}) d\mdp^{'}}. % \nonumber \\
	% &= \frac{\Pr_\mdp(r_{t+1}, s_{t+1}|s_t ,a_t)\bel_t(\mdp)}{\expect_{\bel_t}[\Pr_\mdp(r_{t+1}, s_{t+1}|s_t ,a_t)]}.
	\label{eq:belief-update}
	\end{align}
	Now, we define the Bayesian value function $\util$ analogously to the MDP value function:
	\begin{equation}
	\util^\pol_\bel(s) \defn \int_{\mathcal{M}} V_\mdp^\pol (s) \bel(\mdp) d\mdp.
	\label{eq:bayes-value}
	\end{equation}
	Bayesian value function is the average utility that the decision maker is expected to obtain given its current belief $\bel$ and policy $\pol$ for selecting future actions. A policy computed using Bayesian value function can in general be adaptive, and indeed this holds for the Bayes-optimal policy.
	For completeness, we also define the Bayes-optimal utility $\util_\bel^{*}(s)$, i.e. the utility of the Bayes-optimal policy.
	\begin{equation}
	\util_\bel^{*}(s) \triangleq \max_{\pol \in \Pi} \int_{\mathcal{M}} V_\mdp^\pol (s) \bel(\mdp) d\mdp.
	\label{eq:root-sampling}
	\end{equation}
	% Let $\expect$ denote expectation under any appropriate joint distribution.

	It is well known that by combining the original MDP’s state $s_t$ and belief $\bel_t$ into a hyper-state $\omega_t$, we obtain another MDP called the Bayes Adaptive MDP (BAMDP).
	The optimal policy for a BAMDP is the same as the Bayes-optimal policy for the corresponding MDP.
	\begin{definition}[BAMDP]
		A Bayes Adaptive Markov Decision Process (BAMDP) $\tilde{\mdp} \defn (\Omega, A, \nu, \tau)$ is a representation for an unknown MDP $\mu = (S,A,P,R)$ with a space of information states $\Omega = S \times \mathfrak{B}$, where $\mathfrak{B}$ is an appropriate set of belief distributions on $\mathcal{M}$. %and $S,A$ are the common state and action sets of all $\mdp \in \mathcal{M}$.
		At time $t$, the agent observes the information state $\omega_t = (s_t ,\bel_t )$ and takes action $a_t \in A$. We denote the transition distribution as $\nu(\omega_{t+1}|\omega_t ,a_t )$, the reward distribution as $\tau(r_{t+1} | \omega_t ,a_t)$, and $A$ as the common action space.
	\end{definition}
	For each $s_{t+1}$, the next hyper-state $\omega_{t+1} = (s_{t+1},\bel_{t+1})$  is uniquely determined since $\bel_{t+1}$ is unique given $(\omega_t,s_{t+1})$ and can be computed using equation (\ref{eq:belief-update}). Therefore the information state $\omega_t$ preserves the Markov property. This allows us to treat the BAMDP as an infinite-state MDP with $\nu(\omega_{t+1}|\omega_t ,a_t )$, and $\tau(r_{t+1} |\omega_t,a_t )$ defined as the corresponding transition and reward distributions respectively. The transition and reward distributions are defined as the marginal distributions
	\begin{align*}
	\nu(\omega_{t+1}|\omega_t ,a_t ) \defn \int_{\mathcal{M}} \Pr_\mdp(s_{t+1}|s_t ,a_t)\bel_t(\mdp) d\mdp,\\
	\tau(r_{t+1} |\omega_t,a_t) \defn \int_{\mathcal{M}} \Pr_\mdp(r_{t+1}|s_t ,a_t)\bel_t(\mdp) d\mdp.
	\end{align*}
	Though the Bayes-optimal policy is generally adaptive in the original MDP, it is Markov with respect to the hyper-state of the BAMDP. In other words, $\omega_t$ represents a sufficient statistic for the observed history.

	Since the BAMDP is an MDP on space of hyper-states, we can use backwards induction (alias, value iteration) starting from the set of terminal hyper-states $\Omega_T$ and proceeding backwards to $T - 1,\ldots, t$ following
	\begin{align}
	V_t^{*}(\omega) = \max_{a \in A} \expect[r \mid \omega, a] + \gamma \sum_{\mathclap{\omega^{'}\in\Omega_{t+1}}} \nu(\omega^{'}|\omega,a) V^{*}_{t+1}(\omega^{'}),
	\label{eq:backwards-induction}
	\end{align}
	where $\Omega_{t+1}$ is the reachable set of hyper-states from hyper-state $\omega_t$.
	Equation~\eqref{eq:root-sampling} implies Equation~\eqref{eq:backwards-induction} and vice-versa, i.e. that $\util^*_\bel(s) = V^*_0(\omega)$ for $\omega = (s, \bel)$ (Appendix B). Hence, we can obtain Bayes-optimal policies through backwards induction. Due to the large hyper-state space, this is only feasible for small horizons $T$ in practice.
	% \footnote{Infinite horizon problems can be approximately solved by expanding (or looking ahead) the state-space tree to some finite depth and using bounds on the value function at leaf nodes~\cite{Abbeel2011}. For BAMDP, this approach involves starting with a state $\omega_0$ and building a belief tree to horizon $\horizon$ using equation \eqref{eq:backwards-induction}, as shown in Algorithm \ref{algo1} and illustrated in Figure \ref{fig:tree1}. This gives the $\horizon$-horizon Bayes-optimal policy. This method can be extended to infinite state spaces through the sparse-sampling approximation \cite{DBLP:conf/ijcai/KearnsMN99}, as discussed in~\cite{duff2002olc,wang:bayesian-sparse-sampling:icml:2005}.}

	\begin{algorithm}[t!]
		\caption{FHTS (Finite Horizon Tree Search)}
		\begin{algorithmic}
			\STATE {\bfseries Parameters:} Horizon $\horizon$
			\STATE {\bfseries Input:} current hyper-state $\omega_h$ and depth $h$.
			\IF{$h=\horizon$}
			\STATE return V($\omega_h$) = 0
			\ENDIF
			\FOR{all actions a}
			\FOR{all next states $s_{h+1}$}
			\STATE $\bel_{h+1}$ = UpdatePosterior($\omega_h$,$s_{h+1}$,a)  (eq.~\ref{eq:belief-update})
			\STATE $\omega_{h+1}$ = ($s_{h+1},\bel_{h+1}$)
			\STATE $V(\omega_{h+1}) = \textrm{FHTS}(\omega_{h+1},h+1)$%\dbcomment{What is Algo 1 here? << its this algrithm itself}
			\ENDFOR
			\ENDFOR
			\STATE Q($\omega_h,a$) = 0
			\FOR{all $\omega_{h+1},a$}
			\STATE $Q(\omega_h,a) \mathrel{+}= \nu(\omega_{h+1}|\omega_h,a)\times V(\omega_{h+1})$	\ENDFOR
			\STATE return $\max_a Q(\omega_h,a)$
		\end{algorithmic}
		\label{algo1}
	\end{algorithm}

	%%%%%%
	\subsection{Related Work}
	\label{sec:lit}
	BRL was initially investigated in~\citep{silver1963markovian,martin1967bayesian}. The problem of computational intractability of the Bayes-optimal solution motivated researchers to design approximate techniques. \cite{dearden98bayesian,dearden99bayesian} proposed Bayesian Q-learning and %\cite{duff2002olc} gives one of the most comprehensive introduction to the subject and perhaps the only analytical approach to solving this problem
	\cite{duff2003diffusion} proposed a diffusion based approximation of Bayesian Markov chains. A vast research has been conducted towards model based BRL algorithms, which is comprehensively compiled in a survey by
	\cite{ghavamzadeh2015bayesian}. We classify these algorithms in two categories: Myopic and Lookahead.

	\textbf{Myopic:} Myopic algorithms do not lookahead in future, rather they take actions depending on present information. Thompson sampling~\citep{thompson1933lou} maintains a posterior distribution over transition models, samples an MDP and chooses the optimal policy for the sample. A reformulation of this for BRL is proposed as Bayesian DP in~\citep{strens2000bayesian}. The Best Of Sampled Set (BOSS)~\citep{Asmuth:BOSS} algorithm generalizes this idea to a multi sample optimistic approach. Monte-Carlo Utility Estimates for BRL (MCBRL)~\citep{dimitrakakis2011robust,dimitrakakis:gbrl} generalizes these ideas to lower bound policies and gradient based value function estimates for improved performance.

	\textbf{Lookahead:}
	The simplest algorithm is to calculate and solve the BAMDP up to some horizon $\horizon$, as outlined in \emph{Algorithm~\ref{algo1}} and is illustrated in Figure~\ref{fig:tree1}. A simple modification to it is Sparse sampling by \cite{DBLP:conf/ijcai/KearnsMN99}, which instead only iterates over a set of sampled states. When applied to BAMDP belief tree\footnote{We freely use the term `tree' or `belief tree' to denote the planning tree generated by the algorithms in the hyper-state space of BAMDP.}, the Kearns algorithm would still have to consider all primitive actions. \cite{wang:bayesian-sparse-sampling:icml:2005} improved upon this by using Thompson sampling to only consider a subset of promising actions. High branching factor of the tree still makes planning with deep horizon computationally expensive. Thus, more scalable algorithms, such as BFS3 \citep{asmuth2011approaching} and BAMCP \citep{guez2012efficient}, were proposed.
	Similar to \citep{wang:bayesian-sparse-sampling:icml:2005}, BFS3 also selects a subset of actions but with an optimistic action selection strategy, though the backups are still performed using Bellman equation. BAMCP takes a Monte-Carlo approach to sparse lookahead in belief-augmented version of Markov decision process. BAMCP also uses optimism for action selection.%, which can sometimes yield suboptimal strategy~\cite{coquelin2007bandit}.
	Unlike BFS3, the next set of hyper-states are sampled from an MDP sampled at the root\footnote{Note that ideally the next observations should be sampled from the $P(s_{t+1}|\omega_t)$ instead of $P(s_{t+1}|\omega_{t_o})$, i.e. the next-state marginal at the root belief.}. Since posterior inference is expensive for any non-trivial belief model, BAMCP further applies lazy sampling and rollout policy, inspired by their application in tree search problems~\cite{kocsis2006bandit}.

	\textbf{Our contribution:}
	Unlike other approaches, we focus on reducing the branching factor by considering $\steps$-step policies instead of primitive actions when planning. These policies are generated through (possibly approximate) Thompson sampling. This approach is rounded by using Sparse sampling~\citep{DBLP:conf/ijcai/KearnsMN99}. The reduced branching allows us to build a deeper tree. The intuition why this might be desirable is that if the belief changes slowly enough, an adaptive policy that is constructed out of a tree of $\steps$-step stationary policies will still be approximately optimal. This intuition is supported by the theoretical analysis in Section~\ref{sec:theory}. In Section~\ref{sec:theory}, we prove that our algorithm results in nearly-optimal planning under certain mild assumptions regarding the belief. Section~\ref{sec:experiment}  experimentally shows that we get better policies than the state-of-the-art with less computation time.  The freedom to choose a policy generator allows the algorithm scale smoothly. We choose Policy Iteration (PI) and a variant of Real Time Dynamic Programming (RTDP) for different sizes of environments.

	\section{DEEPER \& SPARSER SAMPLING (DSS)}\label{sec:algo}
	The core idea of DSS algorithm is to plan in the belief tree, not at the individual action level, but at the level of $\steps$-step policies. Figure~\ref{fig:tree2} illustrates this concept graphically.
	At each time-step $t$, Algorithm~\ref{alg:dss} is called with the current state $s$ and belief $\bel$ as input, with additional parameters controlling how the tree is approximated. The algorithm then generates the tree and calculates the value of each policy candidate recursively (for $\finalStage$ stages or episodes), in the following manner:
	\begin{enumerate}
		\item Line 6: Generate $\npolicies$ MDPs from the current belief $\bel_t$, and for each MDP $\mdp_i$ use the policy generator $\generator:\mdp \rightarrow \pol$ to generate a policy $\pol_i$. This gives a policy set $\Pi_\bel$ with $|\Pi_\bel| = \npolicies$.
		\item Line 10-18: Run each policy for $\steps$ steps, collecting total $\steps$-step discounted reward $R$ in BAMDP. Note that we sample the reward and next-state from the marginal (Line 13-14), and also update the posterior (Line 16).
		\item Line 19-21: Make recursive call to DSS at the end of $\steps$ steps. Repeat the process just described for $\nsamples$ times. This gives an $\nsamples$-sample estimate of that policy's utility $\util_\bel^\pol$.
	\end{enumerate}
	Note that the fundamental control unit that we are trying to find here is a policy, hence Q-values are defined over ($\omega_t$,$\pi$) tuples. Since we now have policies at any given tree node, we re-branch only after running those policies for $\steps$ steps. Hence we can increase the effective depth of the belief tree upto $HK$ for the same computational budget. This allows for deeper lookahead and ensures that the approximation error propagated is also smaller as the error is discounted by $\gamma^{HK}$ instead of $\gamma^{H}$.
	We elaborate this effect in the theoretical analysis.

	\begin{algorithm}[t!]
		\caption{DSS}
		\begin{algorithmic}[1]
			\STATE {\bfseries Parameters:} Number of stages $\finalStage$, steps $\steps$, no. of policies $\npolicies$, no. of samples per policy $\nsamples$, policy generator $\generator$
			\STATE {\bfseries Input:} hyper-state $\omega_h=(s_h,\bel_h)$,  depth $h$.
			\IF{$h=\steps \finalStage$}
			\STATE return $V(\omega_h)$ = 0
			\ENDIF
			\STATE $\Pi_{\bel_h} = \{\generator(\mdp_i)| \mdp_i \backsim \omega_{h}, i \in \mathbb{Z}, i \leq \npolicies \}$
			\FOR{all $\pi \in \Pi_{\bel_h}$}
			\STATE $Q(\omega_h,\pi) = 0$
			\FOR{$1$ to $\nsamples$}
			\STATE $R=0,c=\gamma^h,k=0$
			\STATE $\omega_k=\omega_h,s_k=s_h,\bel_k=\bel_h, a_k=\pi(s_h)$
			\FOR{$k = 1, \ldots, \steps$}
			\STATE $s_{k+1} \backsim \nu(\omega_{k+1}|\omega_k ,a_k ) $ %\int_{\mathcal{M}} P_\mdp(s_{t+1},r_{t+1}|s,a)  \dd \bel(\mdp)$
			\STATE $r_{k+1} \backsim \tau(r_{k+1} | \omega_k ,a_k)$
			\STATE $R \pluseq c \times r_{k+1}$; $c = c \times \disc$
			\STATE $\bel_{k+1} = \textrm{UpdatePosterior}(\omega_{k}, s_{k+1},a_k)$ (from eq.~\ref{eq:belief-update})
			\ENDFOR
			\STATE
			$Q(\omega_h, \pol) +\!= R + \textrm{DSS}(\omega_\steps,h+\steps)$
			\ENDFOR
			\STATE $Q(\omega_h,\pi) /= \nsamples$
			\ENDFOR
			\STATE \textbf{return} $\argmax_{\pi} Q(\omega_h,\pi)$
		\end{algorithmic}
		\label{alg:dss}
	\end{algorithm}

	\section{THEORETICAL ANALYSIS}\label{sec:theory}
	The fundamental analysis of~\cite{DBLP:conf/ijcai/KearnsMN99} for any approximate tree based planning algorithm (like Algorithm~\ref{algo1}) is due to union bound on sampling approximation of every action-value at each node in the tree, where bound is obtained due to discounting of error with increasing depth. In reality, due to exponential nodes with $|A||S|$ branching per level, computational limit is quickly reached and leaf-approximations are needed. We improve on this approach by imposing certain assumptions about the belief in the planning tree and using the duality between Eq.\eqref{eq:root-sampling} and Eq.\eqref{eq:backwards-induction}.

	In order to prove that DSS is nearly optimal, we need two assumptions, and consider an idealized version of the algorithm, ignoring some approximations done for computational simplicity.\footnote{In particular, the sampled policies are not strictly coming from the Thompson sampling distribution, due to the use of partial policy iteration or RTDP.}

	\begin{assumption}
		The belief $\bel_h$ in the planning tree is such that $\epsilon_h \leq \epsilon_0/h$, where $h \geq 1$, $\epsilon_h = ||\hat{\bel}_h - \bel_h||_1$ and $\hat{\bel}_h$ is the constant belief approximation at the start of episode $h$.
		\label{ass:approximate-belief}
	\end{assumption}
	The first assumption states that as we go deeper in the planning tree, the belief error reduces. The intuition is that if the belief concentrates at a certain rate, then so does error of Bayes utility for any Markov policy, by the virtue of its definition (shown in Appendix A, Lemma \eqref{lem:stage-error}). The $\epsilon_0$ denotes a constant dependent on the current root belief $\bel$.
	\begin{assumption}
		$\bel_t(\mdp)\bel_t(\mdp') \leq \frac{C}{\MDPdist{\mdp}{\mdp'}}$, where
		$\MDPdist{\mdp}{\mdp'} \defn \max_{s,a} \pnorm{P_\mdp(\cdot \mid s, a) -P_{\mdp'}(\cdot \mid s, a)}{1}$.
		\label{ass:belief-correlation}
	\end{assumption}
	The second assumption states that belief correlation across similar MDPs is higher than across dissimilar ones.

	Our algorithm finds a near-Bayes-optimal policy, as stated in Theorem~\ref{thm:pacbamdp}.
	\begin{theorem}
		Under Assumptions~\ref{ass:approximate-belief} and~\ref{ass:belief-correlation},  $\forall s \in S$
		\begin{align*}
		\DSutil(s) &\geq \util^*_\bel(s) - \left( 2\epsilon_0 \steps \ln  \frac{1}{1 - \disc^{\steps}} + \frac{2(\steps C + \disc^\steps)}{(1 - \disc)}\right)\\
		&\qquad\qquad - \sqrt{\frac{\ln M/\delta}{2N (1 - \gamma)^{2}}}
		\end{align*}
		with probability $1 - \delta$. Here, $\horizon$ is the horizon, divided by parameter $\steps$ into $\finalStage$ stages, i.e, $T=\steps H$. In addition, at each node of the sparse tree, we evaluate $\npolicies$ policies for $\nsamples$ times.
		\label{thm:pacbamdp}
	\end{theorem}
	At the same time, the algorithm is significantly less computationally expensive than basic Sparse sampling~\citep{DBLP:conf/ijcai/KearnsMN99} which would take $O((|A|M)^T)$ calls to the generative BAMDP model, while we require only $O((NM)^{T/K})$ calls for a $T$-horizon problem.
	%Most importantly, the simple sparse sampling algorithm is so inefficient in practise, that it can't be used even for the smallest %of environments.
	%\cdcomment{We are not testing basic sparse sampling, are we? So, I ommit the assertion that it is inefficient in practice}
	\subsection{Proof Overview}
	Let's consider the planning process to be computed till horizon $\horizon$, which is divided in $\finalStage$ episodes each of length $\steps$. Thus, we get $\horizon = \steps \finalStage$.
	Let $\Pi_\steps$ be the set of all policies $\pol_1^\finalStage \defn \{\pol_i\}_{i=1, \ldots, \finalStage}$. Each $\pol_1^\finalStage$ is a concatenation of $\finalStage$, $\steps$-horizon policies. Hereafter, we refer to such policies as $\steps$-step policies. Since planning is divided into episodes, we define the episodic utility in episode $h+1$ as:
	$$\util^\pol_{\bel_h}(s) \defn \int_{\mathcal{M}} V_{0,\steps}^{\pol,\mdp} (s) \bel_h(\mdp) d\mdp.$$
	Here, $\bel_h$ is the belief at start of episode $h$. Similar to the definition of overall utility in Equation~\eqref{eq:bayes-value},  episodic utility of $\pol$ defines the expected utility of taking $\steps$ steps in the BAMDP starting from belief $\bel_h$ .
	Let $\BOpol$ be the Bayes-optimal policy, $\DSpol$ be the DSS policy, $\kBOpol$ the Bayes-optimal adaptive policy that is restricted to $\steps$-step policies, and $\TSpol$ the Thompson sampling policy, with respective utilities $\BOutil, \DSutil, \kBOutil, \TSutil$.

	Now, we write the Bayesian regret of DSS policy relative to the Bayes-optimal policy and decompose it in terms of relative regret of the the aforementioned policies:
	\begin{align}
	\linf{\BOutil - \DSutil}
	\!\!\!\!{=}
	\linf{\BOutil - \kBOutil
		+ \kBOutil - \TSutil
		+  \TSutil - \DSutil}  \notag \\
	\leq  %\sum_{h} \linf{\util^*_{\bel_h} - \util^{\pol_h^{h+1}}_{\bel_h}}
	\linf{\BOutil - \kBOutil}
	\!\!\!\!+ \sum_h   \disc^{\steps h} \linf{\kBOutilH - \TSutilH}
	\!\!\!\!+  \linf{\TSutil - \DSutil}
	%\notag\\
	%  + \disc^{\horizon} / (1 - \disc)
	\label{eq:proof}
	\end{align}
	We bound the first and second term of \eqref{eq:proof} by Lemmas \ref{lemma:anytime} and \ref{lem:ts_regret} below.

	\if \longver 0
	\begin{lemma}[Anytime Error]
		\label{lemma:anytime}
		Under Assumption~\ref{ass:approximate-belief},
		\[
		\linf{\BOutil - \kBOutil} \leq  2\epsilon_0 \steps \ln  \frac{1}{1 - \disc^{\steps}}.
		% \linf{\BOutil - \kBOutil} \leq  2 \epsilon_0 \steps \log(\finalStage)
		\]%\vspace*{-2.5em}
	\end{lemma}

	\begin{lemma}[Error of Thompson-sampling-distributed Policy]
		\label{lem:ts_regret}
		For any episode belief $\bel$, under Assumption~\ref{ass:belief-correlation}:
		\[
		\linf{\kBOutil - \TSutil} \leq \frac{2(\steps C + \disc^\steps)}{(1 - \disc)}.
		\]
	\end{lemma}
	\begin{proof}[Theorem~\ref{thm:pacbamdp} (sketch)]
		Merging the errors due to Anytime error and Thompson-sampling-distributed error from Lemmas \eqref{lemma:anytime} and \eqref{lem:ts_regret}, we obtain $\util^{\TSpol}_\bel(s) \geq \util^*_{\beta}(s) - \left(2\epsilon_0 \steps \ln  \frac{1}{1 - \disc^{\steps}} + \frac{2(\steps C + \disc^\steps)}{(1 - \disc)}\right)$ for all $s$.
		Combined with an additional Hoeffding inequality for last term of eq.\eqref{eq:proof} we obtain Theorem \eqref{thm:pacbamdp}.
	\end{proof}
	\fi
	\section{EXPERIMENTAL ANALYSIS}\label{sec:experiment}
	% \cdcomment{Did you end up not doing any lookahead at all? Then you are simply finding the Bayes-optimal stationary policy in each step. How often are you calculating a new policy to act with? At every step?} Ans: No, not always. Yes, I calculate a new policy to act with at every step, take action in real environment, then redo the computation. Same as other algorithms. we can add the table of optimal parameters here or in appendix.

	\paragraph{Experimental protocol.}
	We empirically evaluate performance of DSS in comparison with three different algorithms on four different environments. We give additional plots in Appendix D, using Python API of our code\footnote{https://github.com/revorg7/DeepSparseSampling}.\\
	Each algorithm has a number of hyperparameters to choose. Some of which, such as the prior belief, are common to all algorithms. The remainder are unique to each algorithm which are tuned in the following manner:\\
	For each environment $\mdp$ and algorithm $\pol$ combination, we evaluate the algorithm's hyperparameter $\lambda$ over $10$ experiments with horizon $T$ and select the value maximising average cumulative reward over them, i.e. $\lambda^*(\mdp, \pol) = \argmax_\lambda \sum_{i=1}^{10} \sum_{t=1}^T r_t^{(i)}$, where $r_{t}^{(i)}$ is the reward sequence of the $i$-th experiment. The parameter sets for each algorithm are detailed in Appendix C.
	The final evaluation, and results shown, was performed over $100$ runs using the chosen $\lambda^*$. This is done to avoid selection of the best parameter in hindsight.
	% More on cross-validation in Appendix C.

	\paragraph{Algorithms.} In our experiments, we consider four lookahead algorithms, all of which expand the BAMDP to a finite horizon.\footnote{Myopic algorithms like Thompson sampling were not excluded. In particular, TS  is a special case of \textit{SBOSS}, but in our hyperparameter search it was always automatically excluded.}

	% \begin{enumerate}[leftmargin=*]
	\textit{Sparser}: Alg.~\ref{alg:dss} with two variants of policy generators: PI and RTDP. PI refers to exact discounted Policy-iteration while RTDP refers \cite{barto1995learning}, where the RTDP horizon can intuitively be taken as $\steps$ as we run the generated policy for next $\steps$-steps in the belief tree.

	\textit{BAMCP}\footnote{\label{foot:implementation}We use the implementations from BAMCP paper.}: The current state-of-the-art. It applies UCT algorithm in belief tree, combined with root-sampling and lazy sampling for faster computation. \citep{guez2012efficient}

	\textit{SBOSS}\textsuperscript{\ref{foot:implementation}}: A more effective variant of BOSS algorithm. BOSS algorithm samples multiple MDPs from the belief, creates an extended MDP using the samples, then solving it to yeild an optimistic policy.~\citep{castro2010smarter}

	\textit{BFS3}\textsuperscript{\ref{foot:implementation}}: An optimistic follow-up to \cite{wang:bayesian-sparse-sampling:icml:2005}, it performs optimistic action selection in belief tree planning. It main advantage lies in non-uniform trajectory selection.~\citep{asmuth2011approaching}
	%\end{enumerate}

	\paragraph{Environments.} We evaluate on the following environments:
	\begin{enumerate}[leftmargin=*]
		\item \textit{Chain}: An MDP consisting of 5 states, connected in a linear chain, with a big reward opposite to the start state at one corner \citep{dearden98bayesian}.\footnote{Note that Chain was not compared in the BAMCP paper. For all other environments, we used a configuration identical to experiments in \cite{guez2012efficient}.}
		\item \textit{DoubleLoop}: A 9-state MDP consisting of two seperate loops, sharing one state in common~\citep{dearden98bayesian}.
		\item \textit{Grid}: Two sparse-reward environments, represented by square grids, of size 5x5 (Grid5) and 10x10 (Grid10), with reward only at goal state. Initial state is always diagonally opposite to the goal state.
		\item \textit{Maze}: A grid world with 264 states, consisting of flags to be collected a various locations, which inturn decide the reward value when goal state is finally reached. The states are encoded by location of agent, as well as flag status~\citep{dearden98bayesian}.
	\end{enumerate}
	\begin{figure*}
		\centering
		\begin{subfigure}[b]{0.32\textwidth}
			\includegraphics[width=\columnwidth]{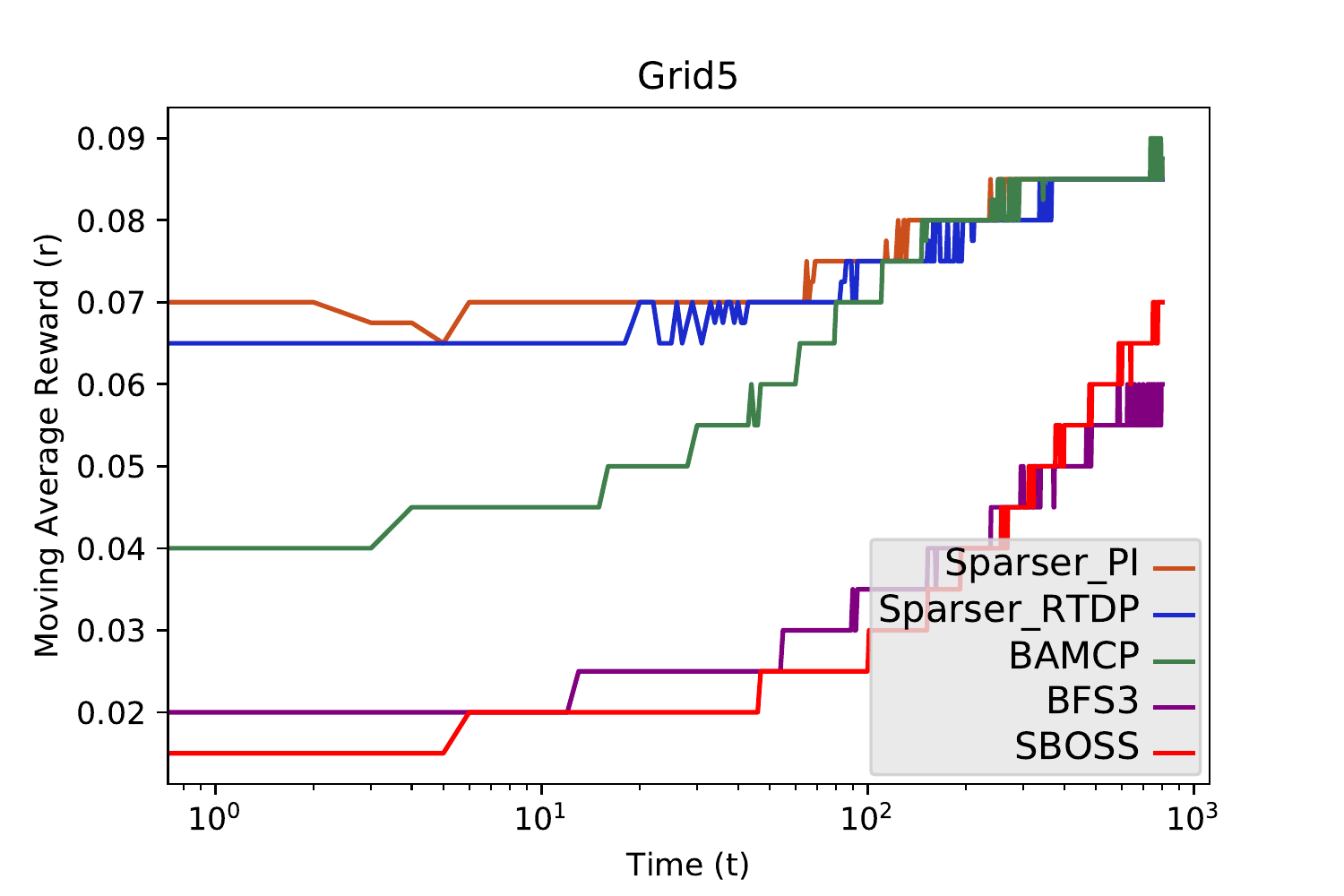}
			% \caption{A gull}
			\label{fig:grid5}
		\end{subfigure}
		~ %add desired spacing between images, e. g. ~, \quad, \qquad, \hfill etc.
		% (or a blank line to force the subfigure onto a new line)
		\begin{subfigure}[b]{0.32\textwidth}
			\includegraphics[width=\columnwidth]{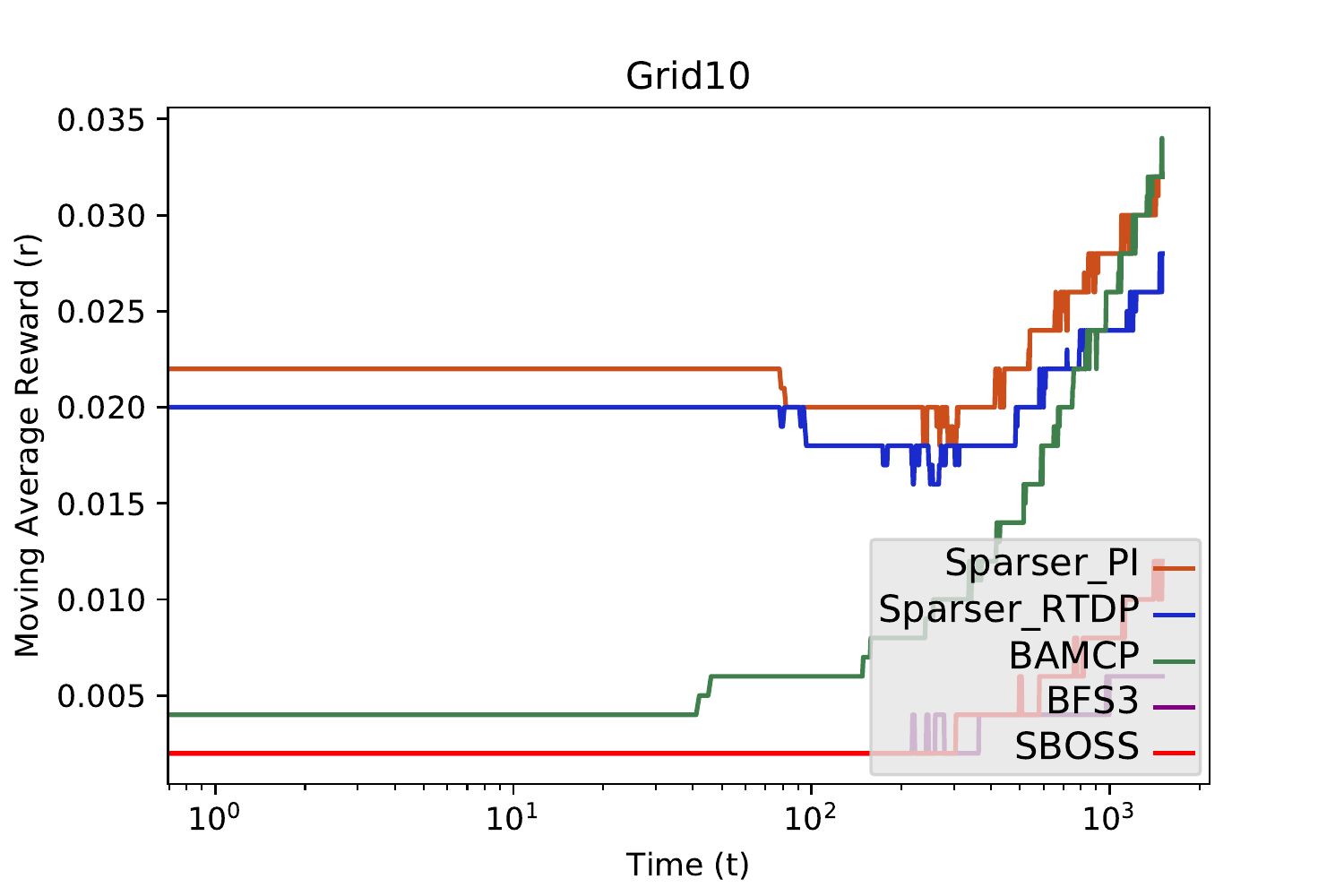}
			% \caption{A tiger}
			\label{fig:grid10}
		\end{subfigure}
		~ %add desired spacing between images, e. g. ~, \quad, \qquad, \hfill etc.
		% (or a blank line to force the subfigure onto a new line)
		\begin{subfigure}[b]{0.32\textwidth}
			\includegraphics[width=\columnwidth]{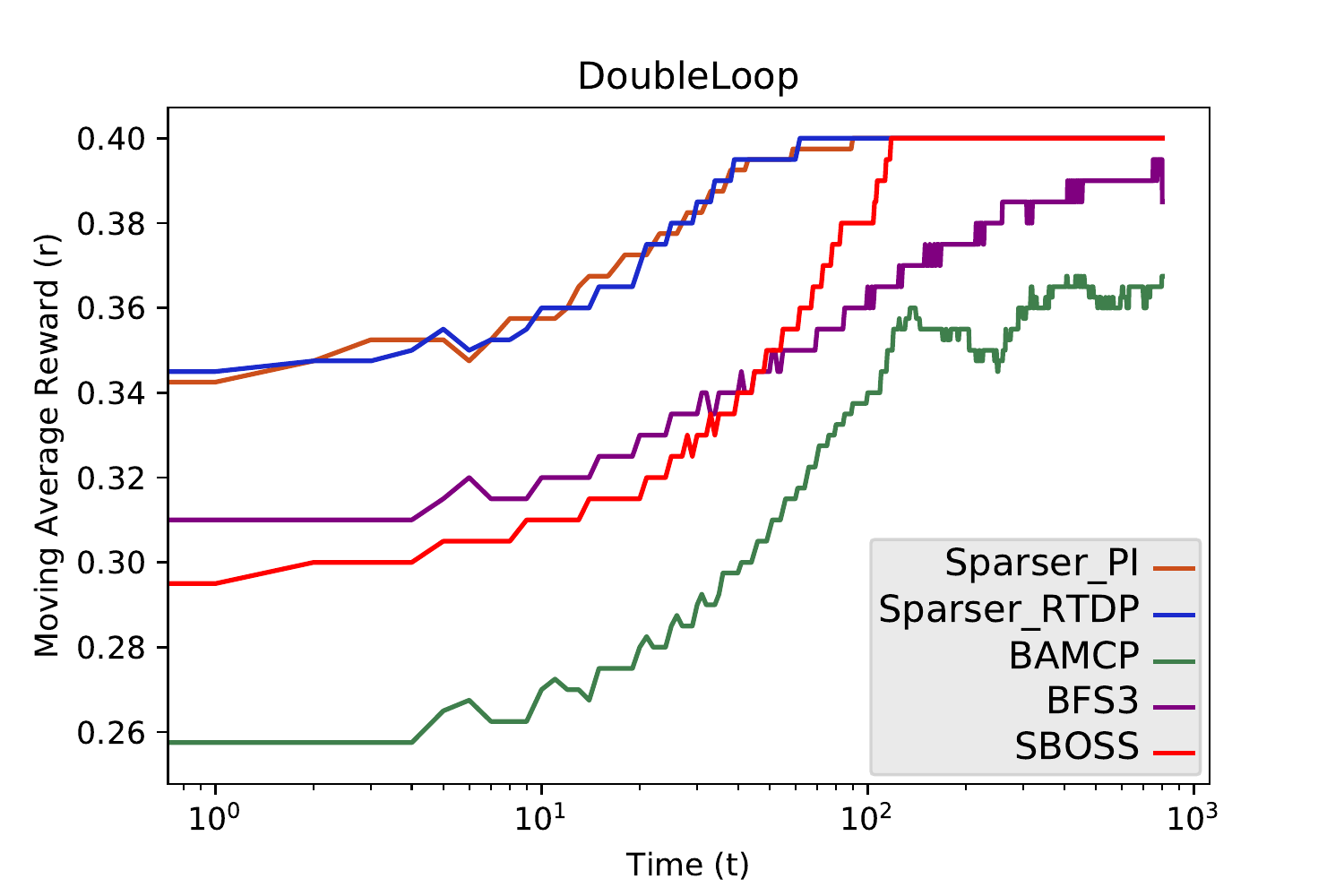}
			% \caption{A mouse}
			\label{fig:loop}
		\end{subfigure}\\
		\begin{subfigure}[b]{0.32\textwidth}
			\includegraphics[width=\columnwidth]{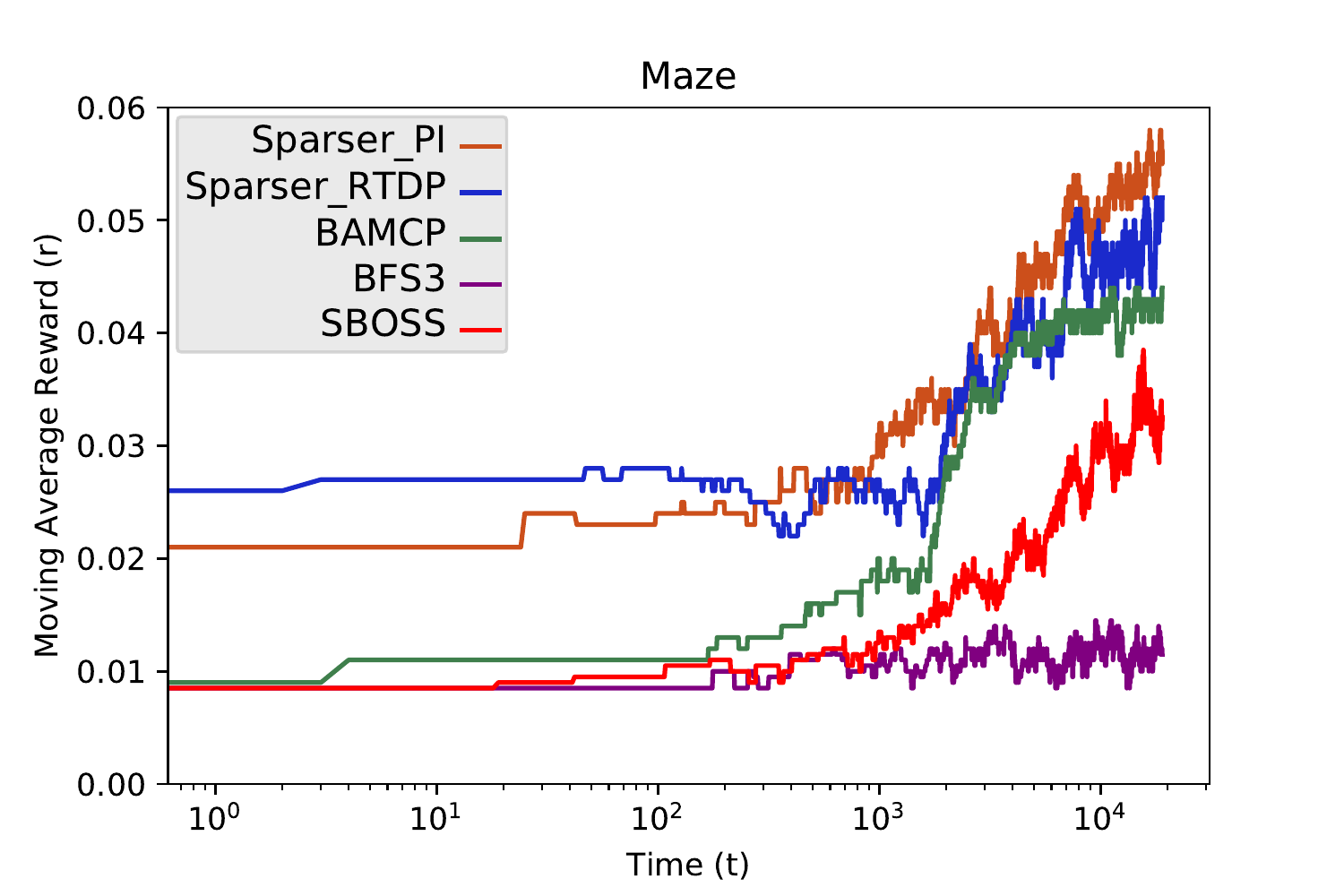}
			% \caption{A mouse}
			\label{fig:maze}
		\end{subfigure}
		\begin{subfigure}[b]{0.32\textwidth}
			\includegraphics[width=\columnwidth]{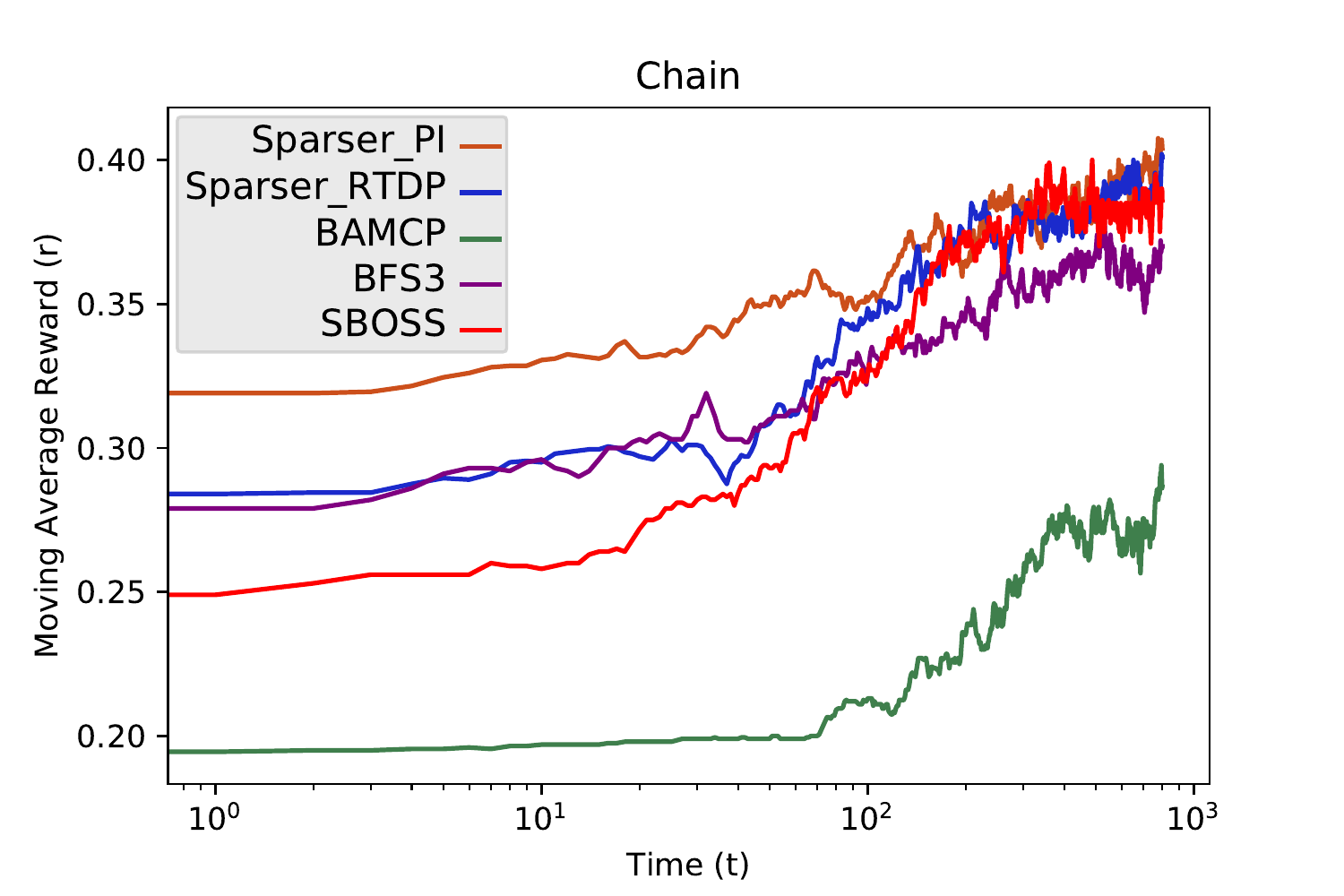}
			% \caption{A mouse}
			\label{fig:chain}
		\end{subfigure}%\vspace*{-1.5em}
		\caption{Moving average performance on log time scale}\label{performance}
	\end{figure*}

	\paragraph{Shared parameters.}
	Some parameters are shared by all algorithms. When possible, we reuse the ones used in~\citep{guez2012efficient}:
	\begin{itemize}
		\item We impose a limit 0.25sec/step for Chain and DoubleLoop, 1.5sec/step for Maze and 1sec/step for the grid environments. Hyperparameter values exceeding those limits were excluded from the hyperparameter search.
		\item We assume known rewards. We recompute the optimal action at each step in simulation.
		\item Experiments last for $T=1000$ steps in Chain, DoubleLoop and Grid5, $T=2000$ in Grid10 steps and in $T=20000$ in Maze.
		\item We use a hierarchical Dirichlet~\citep{friedman1999efficient} on the transition probabilities.
		\item We use the environment simulators from~\citep{guez2012efficient}.\footnote{\label{note1}Code: https://github.com/acguez/bamcp}
	\end{itemize}

	\begin{table*}[!ht]
		\vskip 0.15in
		\begin{center}
			\begin{small}
				\begin{sc}
					\begin{tabular}{lccccr}
						\toprule
						Algorithm & Chain & DoubleLoop & Grid5 & Grid10 & Maze \\
						\midrule
						Sparser-RTDP    & 358.97$\pm$5.15 & \textbf{387.20$\pm$0.64}  & 78.74$\pm$0.80 & 44.32$\pm$0.89 & 849.99$\pm$20.68\\
						Sparser-PI    & \textbf{370.06$\pm$4.71} & 380.60$\pm$0.62  & \textbf{79.01$\pm$0.47} & \textbf{50.91$\pm$0.50} & \textbf{944.99$\pm$19.36}\\
						BAMCP & 267.63$\pm$5.72 & 309.32$\pm$4.26 & 73.92$\pm$0.96 & 37.07$\pm$0.72 & 738.2$\pm$21.96\\
						BFS3    & 340.57$\pm$4.51 & 367.95$\pm$0.74 & 44.94$\pm$0.88 & 8.60$\pm$0.28 & 225$\pm$6.88\\
						SBOSS    & 351.49$\pm$4.28 & 371.11$\pm$1.77  & 47.50$\pm$0.36 & 13.32$\pm$0.35 & 513.25$\pm$5.59\\
						\bottomrule
					\end{tabular}
				\end{sc}
			\end{small}
		\end{center}
		\vskip -0.1in
		\caption{Total reward obtained, averaged over 100 experiments, shown with standard errror.}
		\label{tab:total-reward}
	\end{table*}

	\subsection{Analysis of Results}
	We measure three quantities over 100 trials for each environment:  the mean total reward (Table~\ref{tab:total-reward}), the per-step average reward (Figure~\ref{performance}), and  the CPU time (Table~\ref{TimeTaken}). The CPU time denotes the time taken per episode for the best performing parameters.
	% Quantile plot in Figure~\ref{performance} is omitted for clarity and not required, as we are interested only in the dispersion of cumulative reward at the end of episode.

	Table~\ref{tab:total-reward} shows the average cumulative reward, a comparison metric used in previous works, for each of the algorithms on each of the environments. The standard error incurred by both DSS variants is small enough. This implies that both DSS variants outperform the current state of the art for all environments tested; following a rigorous and unbiased hyperparameter selection process.

	Figure~\ref{performance}, which shows the time evolution of the average reward. For clear illustration, the average reward is smoothed over a window of 200 steps (500 steps for the Maze). DSS outperforms all other algorithms initially, due to better exploratory actions. In most cases there exists at least one (different) algorithm that achieves the asymptotic performance of DSS. This phenomenon is expected since beliefs of all competing algorithms converge to the true model, but they are generally unable to converge as fast as DSS for all the environments.

	The advantage of DSS is not only in terms of performance but also in terms of efficiency. Table~\ref{TimeTaken} shows that DSS takes significantly less time per episode for larger environments than its immediate predecessor, BFS3, and also often manages to outperform the state-of-the-art BAMCP in terms of computation time.

	It is important to note that, although we impose a per-step time limit on computation, the performance of the tested algorithms does not necessarily increase with computation time. For example, we observe that performance of BAMCP actually drops when number of root samples are increased from $10^5$ to $10^6$ while keeping other parameters constant. This reinforces the need for using an unbiased experimental methodology for tuning hyperparameters, as advocated in this paper. Similar observations were made in~\citep{guez2012efficient} regarding SBOSS and BFS3. For DSS, in practice, the performance generally increases with parameters $\npolicies$ and $\nsamples$ but plateaus quite fast. For further reference, the chosen hyperparameters are shown in Table~\ref{optimalParams} (Appendix C).

	\begin{table*}[t!]
		\vskip 0.15in
		\begin{center}
			\begin{small}
				\begin{sc}
					\begin{tabular}{lccccr}
						\toprule
						Algorithm & Chain & DoubleLoop & Grid5 & Grid10 & Maze \\
						\midrule
						Sparser-RTDP  & 0.93 & 1.31  & 5.32 & 97.4 & 1267.2\\
						Sparser-PI    & 2.72 & 1.70  & 5.73 & 142.2 & 1532.0\\
						BAMCP & 0.56 & 1.25 & 172.46 & 315.7 & 1789.4\\
						BFS3    & 6.25 & 2.26 & 54.03 & $>$2000* & 3558.7\\
						SBOSS    & 0.01 & 0.01  & 0.28 & 300.95 & 3695.5\\
						\bottomrule
					\end{tabular}
				\end{sc}
			\end{small}
		\end{center}
		\caption{Time taken in seconds per episode. (*Time limit exceeded)}
		\label{TimeTaken}
	\end{table*}

\section{DISCUSSION AND FUTURE WORK}
% \cdcomment{Not a very nice way to say it.}
%We address the problem of efficient model-based online planning for Bayesian reinforcement learning. We propose an optimism-free algorithm to induce deeper and sparser exploration, with a PAC planning process, and which can achieve state-of-the-art results with lower computational complexity. The main novelty is the use of a candidate policy generator, to generate long-term options in the belief tree, which allows us to create much sparser and deeper trees.

We propose an optimism-free algorithm that induces deeper and sparser exploration, with a PAC planning process, and also achieves state-of-the-art results with lower computational complexity. The PAC guarantee provides DSS with theoretical strength relative to other state-of-the-art algorithms (c.f. Table 4.1 in~\cite{ghavamzadeh2015bayesian}). The analysis also shows how the gap between the Bayes-optimal policy and DSS depends on the main hyperparameter $\steps$.

In comparison, BAMCP is Bayes-optimal policy only asymptotically~\citep{guez2012efficient}. The guarantees for BOSS are PAC-MDP (i.e. that there is only a polynomial number of steps for which its takes an action with unbounded utility error), However,  \citet{bolt} argue that PAC-MDP is not the most suitable for evaluating BAMDP algorithms. Finally, the theoretical properties of BFS3, which can be regarded as the immediate predecessor to DSS, are not known.

Experimental results on different environments show that, compared to the state-of-the-art, our algorithm is both more efficient and obtains higher reward. In practice, we drastically reduce the computation time compared to its immediate Forward Search predecessor BFS3, as can be seen in Table~\ref{TimeTaken}. This is because we only compute policies every $\steps$-step while planning in belief tree. And unlike BFS3, instead of maintaining upper and lower bounds on observation nodes, we simply select them by sampling from the current posterior in the tree branch. %In addition, our algorithm is competitive with the state of the art and it generally achieves significantly higher performance.

Future extensions to this work can be to provide tighter bounds for Thompson policies, similar to very recent work by~\citep{efroni2019multi}; reinforcing this approach of planning at policy level instead of individual action level. DSS could also possibly be extended to continuous state spaces by keeping a prior over models other than discrete MDPs, such as linear state-space model or a non-linear Neural Network model. However, this would require us to strike a delicate balance between approximations in inference and planning, and is left as a subject for future work.

% \newpage

\bibliographystyle{plainnat}
\bibliography{references}
\clearpage
\appendix

\begin{center}
  \Large{Bayesian Reinforcement Learning via Deep, Sparse Sampling}
  \\
  \large{Supplementary material}
\end{center}

\if \longver 0
\section{Proofs of Section~\ref{sec:theory}}

The following Lemma shows how episodic error is directly proportional to the belief-error:
\begin{lemma}[Per-episode Error]
  \label{lem:stage-error}
  The error in reward under the true and approximate belief in any episode is bounded by
  \begin{align*}
    \Delta_h \defn
    \sup_\pol
    \linf{\util^\pol_{\bel_h}
    -
    \util^\pol_{\hat{\bel}_h}}
    \leq
    \frac{1 - \disc^{\steps}}{1 - \disc}
    \epsilon_h
    \leq
    \steps \epsilon_h.
  \end{align*}
\end{lemma}
\begin{proof}
  Due to the fact that the two BAMDPs induced by the beliefs are $\epsilon_k$ close in L1 norm, we can use the argument in Theorem 1 of \cite{dimitrakakis2011robust}.
\end{proof}

We shall also use a trivial lemma from Analysis:
\begin{lemma}
  If $\linf{f - g} \leq \epsilon$ and $f(x^*) \geq f(x)$, $g(y^*) \geq g(y)$ then
  $f(y^* ) \geq f(x^*) - 2 \epsilon$.
  \label{lem:eps-optimality}
\end{lemma}
We are now ready to prove the main results:

\addtocounter{lemma}{-4}
\if \longver 0
\begin{lemma}[Anytime Error]
  %\label{lemma:anytime}
  Under Assumption~\ref{ass:approximate-belief},
  \[
  \linf{\BOutil - \kBOutil} \leq  2\epsilon_0 \steps \ln  \frac{1}{1 - \disc^{\steps}}.
  % \linf{\BOutil - \kBOutil} \leq  2 \epsilon_0 \steps \log(\finalStage)
  \]%\vspace*{-2.5em}
\end{lemma}
\fi
\begin{proof}%[Proof of Lemma~\ref{lemma:anytime}]
  First note that if the belief is only changing every $\steps$ steps, then the Bayes-optimal policy is Markovian over $\steps$ steps. This means that finding a $\steps$-step Markovian policy starting from belief $\bel_h$ is the same as finding the optimal policy for a fixed belief $\hat{\bel}_h$.This allows us to use Lemma~\ref{lem:stage-error} to bound the error of the $\steps$-step policy.

  Let $f(\pol) \defn \util^\pol_{\bel_h}$ and $g(\pol)\defn\util^\pol_{\hat{\bel_h}}$. Therefore,
  \begin{align*}
   % \begin{split}
    \linf{\BOutil - \kBOutil} &= \sum_{h=1}^H \disc^{h \steps} (\BOutilH -  \kBOutilH)\\
    &\stackrel{(a)}{\leq} 2\sum_{h=1}^\finalStage \disc^{h\steps} \Delta_h\\
     &\stackrel{(b)}{\leq} 2\sum_{h=1}^\finalStage  \disc^{h\steps} \steps \epsilon_h \\
      &\stackrel{(c)}{\leq} 2\steps \times \sum_{h=1}^\finalStage  \disc^{h\steps} \epsilon_0 /h \\
      &\leq 2\epsilon_0 \steps \times \sum_{h=1}^\infty  \disc^{\steps h} / h \\
      %&\leq \epsilon_0 \steps \times \sum_{h=1}^\infty  \disc^{\steps h}  \\
      &\stackrel{(d)}{=}  2\epsilon_0 \steps \times \ln \frac{1}{1 - \disc^{\steps}}
   % \end{split}
  \end{align*}
  (a) is obtained using Lemma~\ref{lem:eps-optimality}.
  
  (b) is a consequence of Lemma~\ref{lem:stage-error}.
  
  (c) is a consequence of Assumption~\ref{ass:approximate-belief}.
  
  (d) is derived from the fact that $\sum_{x=1}^{\infty} \frac{a^x}{x} = -\log(1-a)$ for $1 > a \geq 0$.
  
  Therefore, $\linf{\BOutil - \kBOutil} \leq  2\epsilon_0 \steps \ln  \frac{1}{1 - \disc^{\steps}}$.
\end{proof}
%\vspace*{-1em}
%\begin{lemma}[Error of Thompson-sampling-distributed Policy]
%  Let $\TSpol$ denote the Thompson sampling-distributed policy for $\bel$. Under Assumption~\ref{ass:belief-correlation}, the expected regret of Thompson sampling is $O((1 - \disc)^{-2})$.
%  
%  $$\linf{\BOutil - \TSutil} \leq \frac{C}{(1-\gamma)^2}.$$
%\end{lemma}
%\begin{proof}%[Proof of Lemma~\ref{lem:ts_regret}]
%  \begin{align*}
%    \begin{split}
%      &\BOutil(s) - \TSutil (s)\\
%      &= \int_{\mathcal{M}} \dd\bel(\mdp) \left[V_\mdp^{\BOpol} (s)  - \int_{\mathcal{M}} V_{\mdp}^{\pol^*(\mdp')} (s) d \bel(\mdp')\right]\\		
%      &=  \int_{\mathcal{M}} \int_{\mathcal{M}} \left[V_\mdp^{\pol^*_\mdp} (s)  -  V_{\mdp'}^{\pol^*_{\mdp'}} (s) \right] d\bel(\mdp) d\bel(\mdp')\\
%      &\leq  \int_{\mathcal{M}} \int_{\mathcal{M}}  \frac{2 \MDPdist{\mdp}{\mdp'}}{(1 - \disc)^2} d\bel(\mdp) d\bel(\mdp') 
%      \leq
%      \frac{C}{(1 - \disc)^2}.
%    \end{split}
%  \end{align*}
%  % where is second-last line follows from $\util(\pol_\mdp^*, \mdp) = \sum_{\mdp'} \bel(\mdp') \util(\pol_\mdp^*, \mdp)$ and 
%  The last inequality is a direct consequence of Assumption~\ref{ass:belief-correlation}. Hence, we get 
%  \[
%  \linf{\BOutil - \TSutil} \leq \frac{C}{(1-\gamma)^2}.
%    \]
%\end{proof}
% \begin{lemma}  \label{lem:compound-error}
% If $\epsilon_k \leq c^k$ for any $0 \geq c <1$, then
% \[
% \util^*_\beta - \util(\pol^K_1, \beta) \leq  \frac{2 \steps }{1 - \disc^\steps c}
% \]\vspace*{-2.5em}
% \end{lemma}
% \begin{proof}[Proof of Lemma~\ref{lem:compound-error}]
% 
% \end{proof}

\begin{lemma}[Error of Thompson-sampling-distributed\footnote{i.e. the optimal polices of MDPs distributed according to the current belief.} episodic Policy]
Under Assumption~\ref{ass:belief-correlation}%, the expected regret of Thompson sampling is $O(\frac{1}{1 - \disc})$.\\
  $$\linf{\kBOutil - \TSutil} \leq \frac{2(\steps C + \disc^\steps)}{(1 - \disc)}.$$
\end{lemma}
\begin{proof}%[Proof of Lemma~\ref{lem:ts_regret}]
  We want to show that the value of the TS-episodic policy is not much worse than the Bayes-optimal $\steps$-step stationary policy.
  \begin{align*}
    \begin{split}
      &\quad \kBOutil - \TSutil\\
      &\stackrel{(a)}{\leq} \BOutil(s) - \TSutil (s)\\
      &\stackrel{(b)}{\leq}
      \int_{\mathcal{M}} d\mdp \bel(\mdp) \left[V_{0,\steps}^{\BOpol,\mdp} (s)  - \int_{\mathcal{M}} d\mdp' V_{0,\steps}^{\pol^*_{\mdp'},\mdp'} (s)  \bel(\mdp')\right] \\
&+ \frac{2 \disc^\steps}{1 - \disc}
\\		
&=  \int_{\mathcal{M}} \int_{\mathcal{M}} \left[V_{0,\steps}^{\pol^*_\mdp,\mdp} (s)  -  V_{0,\steps}^{\pol^*_{\mdp'},\mdp'} (s) \right] \bel(\mdp) \bel(\mdp') d\mdp d\mdp' \\
&+ \frac{2 \disc^\steps}{1 - \disc}\\
&\stackrel{(c)}{\leq}  \int_{\mathcal{M}} \int_{\mathcal{M}} \frac{2 \steps \MDPdist{\mdp}{\mdp'}}{(1 - \disc)} \bel(\mdp) \bel(\mdp') d\mdp d\mdp'
      \\
      &+ \frac{2 \disc^\steps}{1 - \disc}\\
      &\stackrel{(d)}{\leq}
      \frac{2(\steps C + \disc^\steps)}{(1 - \disc)}.
    \end{split}
  \end{align*}
  (a) follows by the definition of the Bayes-optimal policy. 
  
  (b) follows by truncating the reward sequence to $\steps$ steps. 
  
  (c) follows  from the approximate MDP Lemma~\citep[Lemma 4]{colt03:Even-Dar:MDP-Equivalence} by definition of the MDP distance $\MDPdist{\mdp}{\mdp'}$. 
  
  (d) is a direct consequence of Assumption~\ref{ass:belief-correlation}. 
\end{proof}

\paragraph{Proof of Theorem~\ref{thm:pacbamdp}.}
For the final proof, we add the effect of sampling in the total error from the two lemmas:
\begin{proof}%[Proof of Theorem~\eqref{thm:pacbamdp}]
  Merging the errors due to Thompson-sampling-distributed error and the anytime-error from Lemma~\eqref{lem:ts_regret} and~\eqref{lemma:anytime}, we obtain for all $s$
  \[
    \util^{\TSpol}_\beta(s) \geq \util^*_{\beta}(s) - \left( 2\epsilon_0 \steps \ln  \frac{1}{1 - \disc^{\steps}} + \frac{2(\steps C + \disc^\steps)}{(1 - \disc)}\right).
  \]
  We can then use Hoeffding's bound since utility of $\DSpol$ is just sampled utility of $\TSpol$.
  For simplicity, let $\bar{\rho}$ be the expected error of a TS policy and $\rho_i$ of the $i$-th sampled policy and let \[\varepsilon = \sqrt{\frac{\ln(n/\delta)}{2N}}(1 - \disc)^{-1}.\] Then, we bound the probability of minimal-error policy among samples has an error more than $\varepsilon$ than the expectation:
  \begin{gather*}
    \Pr\left(\min \cset{\rho_i}{i=1, \ldots, N} \geq \bar{\rho} + \varepsilon \right) \\
    \leq \Pr\left(\frac{1}{N} \sum_{i=1}^N \rho_i \geq \bar{\rho} + \varepsilon \right) \leq \delta / n,
  \end{gather*}
  where the last inequality is from Hoeffding, and the boundedness of rewards. Since there are $n$ such policies, and with a union bound, the probability that any policy has an error of more than $\varepsilon$ worse than the expected, is bounded by $\delta$.
\end{proof}

%%% Local Variables:
%%% mode: latex
%%% TeX-master: "aistats20"
%%% End:

\fi
\section{Root sampling and look-ahead view equivalence}
Denote $\expect_\bel$ as the expectation under marginals $\nu$ and $\tau$.
The optimal for Bayesian value function can be calculated by noting the following equivalence relation:
\begin{align}
  &\util_{\bel_t}^{\pi}(s_t) \notag\\
  &= \int_{\mathcal{M}} V_{\mdp}^{\pi}(s_t)\bel_t(\mdp)d\mdp \notag \\
  &= \int_{\mathcal{M}} \int_{r_{t+1}} r_{t+1} \Pr_\mdp (r_{t+1}|s_t,a_t)  \bel_t(\mdp)d\mdp \; dr \notag \\
  &+ \gamma \int_{\mathcal{M}} \sum_{s' \in s_{t+1}}\Pr_\mdp(s'|s_t,a_t) V_\mdp^\pol(s') \bel_t(\mdp)d\mdp \notag\\
  &= \int_{r_{t+1}} \int_{\mathcal{M}} r_{t+1} \Pr_\mdp (r_{t+1}|s_t,a_t)  \bel_t(\mdp)d\mdp \; dr \notag \\
  &+ \gamma \sum_{s' \in s_{t+1}} \int_{\mathcal{M}} V_\mdp^\pol(s') \Pr_\mdp(s'|s_t,a_t) \bel_t(\mdp)d\mdp \label{eq:root-view}\\
  &=\int_{r_{t+1}} \tau(r_{t+1}|\omega_t,a_t) \int_{\mathcal{M}} r_{t+1} \bel_{t+1}(\mdp)d\mdp \; dr \notag \\
  &+ \gamma \sum_{s' \in s_{t+1}} \nu(\omega_{t+1}|\omega_t ,a_t ) \int_{\mathcal{M}} V_\mdp^\pol(s') \bel_{t+1}(\mdp)d\mdp \label{eq:back-view}\\
  &= \int_{r_{t+1}} r_{t+1} \tau(r_{t+1}) dr + \gamma \sum_{\omega_{t+1}} \util^\pol_{\bel_{t+1}}(s_{t+1})  \nu(\omega_{t+1}) \label{eq:martingale} \\
  &= \expect_\bel[r_{t+1} + \gamma V^\pol_{t+1}(s_{t+1},\bel_{t+1})] \tag{Q.E.D}
\end{align}
We obtain Eq.\eqref{eq:back-view} from Eq.\eqref{eq:root-view} using Bayes rule (eq.\eqref{eq:belief-update}) and the fact that marginal distributions $\nu$ and $\tau$ are independent of the next belief $\bel_{t+1}$ (since they are normalization constants in eq.\eqref{eq:belief-update}).\\
Eq.\eqref{eq:martingale} follows from the definition of Bayesian value function and the fact that $\bel_{t+1}$ adds no new information about $r_{t+1}$, i.e, its a martingale.
% \newpage

\section{Parameter Selection for Experiments}
The dependence of lookahead parameter $\steps$ on other parameters could be found by assuming that the loss terms due to parameter $\steps$ being less than due to sampling, i.e., 
	\begin{align}
	&\Rightarrow 2\epsilon_0 \steps \ln  \frac{1}{1 - \disc^{\steps}}  + \frac{2(\steps C + \disc^\steps)}{(1 - \disc)}
	\leq \sqrt{\frac{\ln M/\delta}{2N (1 - \gamma)^{2}}} \notag \\
	&\Rightarrow \frac{2\epsilon_0 \steps}{1 - \disc}  + \frac{2(\steps C + \disc^\steps)}{(1 - \disc)}
	\leq \sqrt{\frac{\ln M/\delta}{2N (1 - \gamma)^{2}}} \label{eq:log-property} \\
	&\Rightarrow 2(\epsilon_0 + C) \steps + 2\disc^\steps \leq \sqrt{\frac{\ln M/\delta}{2N}} \notag \\
	&\Rightarrow \steps \leq (\sqrt{\frac{\ln M/\delta}{8N}} - \disc^\steps)\frac{1}{\epsilon_0 + C} \notag
	\end{align}
Hence $\steps$ needs to grow only inversely to the root of number of policy samples ($\sqrt{\npolicies}$), and its dependence on $\nsamples$ is very slow (only $\sqrt{\ln \nsamples}$).
% \paragraph{Parameter selection:}
\label{sec:tuning}
\begin{table*}[t]
	%\vskip 0.15in
	\begin{center}
		\begin{small}
			\begin{sc}
				\begin{tabular}{lccccr}
					\toprule
					Algorithm & Chain & DoubleLoop & Grid5 & Grid10 & Maze\\
					\midrule
					Sparser-RTDP  & (8,4,10,2) & (4,4,18,2)  & (4,2,50,1) & (4,2,200,2)& (4,4,500,1)\\
					Sparser-PI    & (4,4,5,2) & (4,4,18,2)  & (2,2,25,1) & (2,2,100,2)& (4,2,100,1)\\
					BAMCP & (auto,100) & (15,100) & (50,10000) & (50,10000)& (50,1000)\\
					BFS3    & (10,100) & (10,10) & (10,10) & (5,10)& (5,10)\\
					SBOSS    & (8,3) & (2,3)  & (2,3) & (2,3) & (2,3)\\
					\bottomrule
				\end{tabular}
			\end{sc}
		\end{small}
	\end{center}
	\caption{Best parameters obtained from the initial 10 tuning runs.}
	\label{optimalParams}
\end{table*}

Here we describe the chosen hyperparameters for each algorithm shown in Table~\ref{optimalParams}. For each algorithm, these are:

\begin{enumerate}
\item BAMCP: (depth,no. of simulations): No. of simulations range from $10$ to $10^5$, or until the environment time-limit is reached. Depth is between \{15,50,\textit{auto}\}, using the original implementation.
\item SPARSER : (no.of sampled policies,no.of samples per policy,depth parameter K, Horizon). PI is performed upto 1e-4 accuracy, while RTDP performs lookahead planning of depth 15 for all environments, except larger Grid10 and Maze, where depth is set to 50.
\item BFS3: (branching-factor,no.of simulations). Depth is fixed at 15 for all except larger Grid10 and Maze environment, for which it is 50. Branching factor is between \{5,10,15\} and no. of simulations between \{10,10,1000\}.
\item SBOSS: (no.of samples,sampling threshold) Cross-validated against \{2,4,8,16,32\} and \{3,5,7\} respectively.
\end{enumerate}

\section{Additional plots}

To target larger audience, we provide Python API for the original C++ implementation used in the paper, using Pybind11~\citep{jakob2017pybind11}. It is available at the following link: https://github.com/revorg7/DeepSparseSampling 
\\

This API was used in conjunction with Bsuite environment API by Deepmind~\citep{osband2019behaviour} to draw Regret plots comparing DSS to BDQN (Bootstrapped DQN) and TS (Thompson sampling) in Figure~\ref{fig:regret}.

We did this to promote reusability of DSS, as well as demonstrate reproducibility of DSS's advantage over model-free algorithms such as BDQN for discrete grid-world environments (upto 20x20 atleast).

%\begin{center}
%\includegraphics[width=.3\linewidth]{example-image}\quad
%\includegraphics[width=.3\linewidth]{example-image-a}
%\end{center}

\begin{figure*}[t!]
    \centering
	\includegraphics[width=0.9\textwidth,totalheight=0.4\textheight]{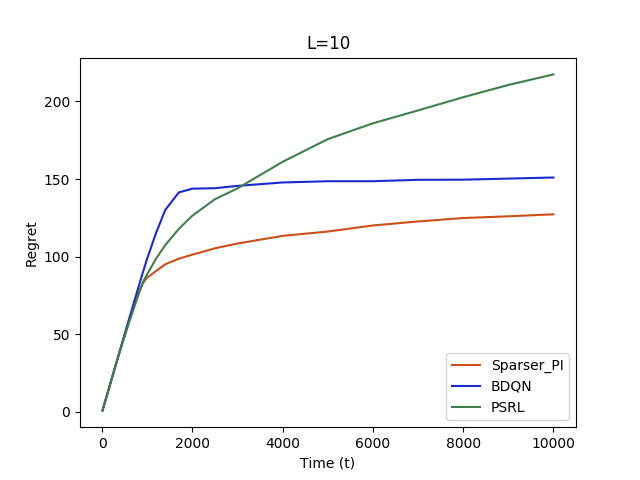}
	\includegraphics[width=0.9\textwidth,totalheight=0.4\textheight]{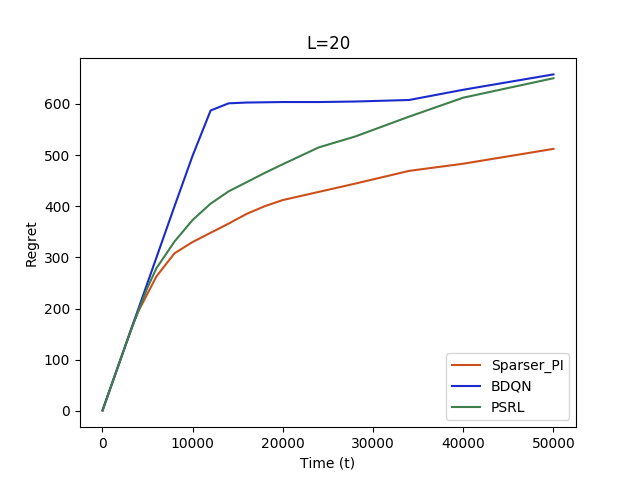}
    \caption{Regret plots(lower is better) for Deep-sea environment for different size parameter 'L'.}
    \label{fig:regret}
\end{figure*}

\end{document}